\documentclass{article}

\usepackage[dblblindworkshop,final]{neurips_2026}
\workshoptitle{SLMs for Agentic Systems, Paris, France, 2026}

\usepackage[utf8]{inputenc}
\usepackage[T1]{fontenc}
\usepackage{hyperref}
\hypersetup{colorlinks=true, linkcolor=red, citecolor=blue, urlcolor=magenta}
\usepackage{url}
\usepackage{booktabs}
\usepackage{amsmath,amssymb,amsfonts,amsthm}
\usepackage{graphicx}
\usepackage{multirow}
\usepackage{enumitem}
\usepackage{algorithm}
\usepackage{algpseudocode}
\usepackage{xcolor}

\newcommand{\frontierref}{Figure~\ref{fig:frontiers}}
\newcommand{\figonewidth}{\linewidth}

\newtheorem{theorem}{Theorem}
\newtheorem{proposition}{Proposition}
\newtheorem{definition}{Definition}
\newtheorem{corollary}{Corollary}
\newtheorem{lemma}{Lemma}

\newcommand{\safeincludegraphics}[2][]{%
  \IfFileExists{#2}{\includegraphics[#1]{#2}}{%
    \fbox{\begin{minipage}[c][0.18\textheight][c]{0.88\linewidth}\centering Missing figure asset:\\ \texttt{\detokenize{#2}}\end{minipage}}%
  }%
}
\newcommand{\carb}{\textsc{CARB}}
\newcommand{\gcarb}{\textsc{G-CARB}}
\newcommand{\E}{\mathbb{E}}

\newcommand{\lstop}{\lambda_{\bot}}

\newcommand{\Lunion}{\ell_{\mathrm{union}}}

\title{G-CARB: Graph-Localized Conformal Agent Risk Budget for Compositional Harm}

\author{%
  \textbf{Zijun Yu}\textsuperscript{*}\\
  \normalfont Department of Mathematics and Statistics\\
  \normalfont McGill University\\
  \normalfont Montr\'eal, QC, Canada\\
  \texttt{zijun.yu@mail.mcgill.ca}
  \And
  \textbf{Yu Gu}\textsuperscript{*}\\
  \normalfont Department of Mathematics and Statistics\\
  \normalfont McGill University\\
  \normalfont Montr\'eal, QC, Canada\\
  \texttt{yu.gu4@mail.mcgill.ca}
  \And
  \textbf{Vahid Partovi Nia}\\
  \normalfont Department of Mathematics and\\
  \normalfont Industrial Engineering\\
  \normalfont Polytechnique de Montr\'eal\\
  \normalfont Montr\'eal, QC, Canada\\
  \texttt{vahid.partovinia@polymtl.ca}
  \And
  \textbf{Masoud Asgharian}\\
  \normalfont Department of Mathematics and Statistics\\
  \normalfont McGill University\\
  \normalfont Montr\'eal, QC, Canada\\
  \texttt{masoud.asgharian2@mcgill.ca}
}

\begin{document}
\maketitle
\begingroup
\renewcommand{\thefootnote}{\fnsymbol{footnote}}
\footnotetext[1]{Equal contribution.}
\endgroup

\begin{abstract}
Small language model (SLM) agents need safety controls that track consequences across tool calls with little monitoring overhead. A private read, for example, becomes a leak when a later action sends that data outside the system. We introduce CARB (Conformal Agent Risk Budget), which calibrates when to stop an agent using a ledger of harm incurred before stopping. Under exchangeable episodes, standard conformal risk control bounds this declared loss in expectation over calibration and a future episode. G-CARB selects scorer evidence along observable dependencies from private sources to outgoing actions. The ledger still covers the entire executed history, and computing the gate score requires no additional language-model inference. On AgentDojo replay with two 14B backbones, G-CARB roughly halves scorer-input records at intermediate risk budgets while improving autonomous task completion relative to full-prefix scoring; random context of the same size achieves similar gains. Controlled examples show how retaining the relevant dependency can further avoid stopping benign work.
\end{abstract}

\section{Introduction}

An agent reads a private certificate, copies a field into a note, and later emails that note outside the organization. Together, these actions leak the private field. Preventing the leak requires connecting the proposed send to the earlier read. Tool-using language models routinely act across such multi-step trajectories \citep{yao2022react,schick2023toolformer,zhou2023webarena,deng2023mind2web}. Prompt-injection studies, agent-safety benchmarks, and workflow audits expose the risks of these linked actions \citep{greshake2023injection,debenedetti2024agentdojo,zhan2024injecagent,andriushchenko2024agentharm,wu2026chain}.

For small language model (SLM) agents, the safeguard should also be inexpensive. We use an external gate: the agent proposes a tool call, a lightweight scorer evaluates the available evidence, and a calibrated threshold determines whether execution continues. The gate tracks dependencies without invoking another language model at every decision.

\begin{figure}[tb]
\centering
\safeincludegraphics[width=\figonewidth]{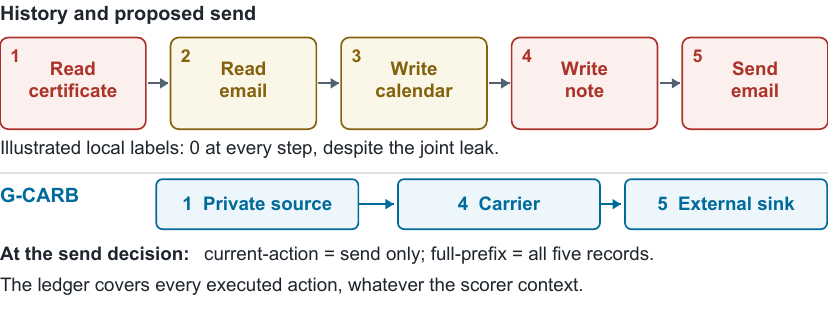}
\caption{\textbf{A leak across tool calls.} The certificate, note, and external send form a private source--carrier--sink chain. \gcarb{} retains this chain; Full-prefix includes unrelated work. The illustrated local labels miss the leak. The executed-step baseline instead labels the first harmful event and divides by the executed action count.}
\label{fig:running}
\end{figure}

Conformal risk control (CRC) bounds the expected loss it is given, so the choice of loss matters. Figure~\ref{fig:running} illustrates how labels on individual actions can miss a leak. Even when the first harmful action is correctly labeled, dividing by many executed actions can make the harmful-step rate small. The probability of harm in an episode can remain high. We call this the \emph{control-unit mismatch} and examine it with an executed-step-rate baseline. CORA represents the broader line of pre-action conformal control \citep{feng2026cora}.

\carb{} (Conformal Agent Risk Budget) evaluates each stopping policy with a global \emph{harm ledger}: a record of declared harmful events across all actions executed before stopping. We replay each policy on logged episodes to obtain its losses, then apply standard CRC to calibrate the stopping threshold. Harm incurred before stopping remains charged.

\gcarb{} uses observable dependencies to select the records supplied to the scorer. This focuses scoring on relevant history while the ledger continues to cover every executed action. We evaluate the gate on two 14B agents in AgentDojo. Controlled examples examine when retaining the relevant dependency adds value beyond shortening context.

\section{Related Work}
Conformal prediction provides finite-sample guarantees under exchangeability; CRC extends this framework to bounded monotone losses \citep{vovk2005algorithmic,angelopoulos2022conformal}. Agent applications differ in what they control. CORA targets harmful executed actions \citep{feng2026cora}, and CROP calibrates how much of a reasoning trace to retain \citep{cheung2026crop}. Other methods attribute failures within completed trajectories \citep{feng2026caea}, score or gate completed runs \citep{tayebati2026tracer,opoku2026toolchain}, or regulate deployment relative to a safe policy \citep{prinster2026conformalpolicy}. Related work calibrates generated sets, uncertain plans, and reasoning graphs \citep{quach2023conformal,han2024uala,ren2023robots,fayyazi2026proof,hittesdorf2026dcf,wang2026itcr}. \carb{} applies CRC to the harm accumulated along the execution prefix a stopping policy permits.

Dependency-based defenses use information flow, graphs, or formal rules to enforce security. Examples include CaMeL, RTBAS, AgentArmor, FORGE, and AgentFlow \citep{debenedetti2025camel,zhong2025rtbas,wang2025agentarmor,palumbo2026forge,wang2026agentflow}. IPIGuard constrains execution through a planned tool-dependency graph \citep{an-etal-2025-ipiguard}; Task Shield checks whether instructions and tool calls contribute to the user's task \citep{jia-etal-2025-task}. In \gcarb{}, the graph supplies evidence to a separate scorer, and calibration uses a ledger defined independently of that graph. AgentDojo provides the evaluation traces \citep{debenedetti2024agentdojo}, within a broader ecosystem of agent-safety benchmarks \citep{ruan2023toolemu,andriushchenko2024agentharm,zhan2024injecagent,yuan2024rjudge,zhang2024asb}. Our experiments compare calibration targets and context selectors within this pipeline.

\section{Problem Formulation}

An ungated episode records an agent's run without a stopping gate. Write it as $Z=(R_1,\ldots,R_T)$, with one record $R_t=(A_t,O_t)$ per executed action. Here $A_t$ is the action at step $t$, $O_t$ is its returned observation, and $T$ is the episode length. A later action may use any earlier observation, so records within an episode may be arbitrarily dependent.

Before $A_t$ executes, a fixed selector $C$ chooses evidence $C_t=C(Z_{<t},A_t)$ from the history and the proposal. A fixed scorer assigns $q_t=s(C_t)\in[0,1]$. The policy $\pi_\lambda$ executes $A_t$ when $q_t\leq\lambda$ and stops before the first proposal with $q_t>\lambda$, handing control to a human or a sandbox. The resulting executed prefix is $P_\lambda(Z)$. Larger thresholds permit longer prefixes; a stop-all threshold $\lstop$ permits no actions. At decision time, the gate has access only to the history and proposal, since $O_t$ becomes available after execution.

We evaluate harm across the full executed prefix. For a data leak, this means checking whether an outgoing action carries data from an earlier private read. The ledger defines these events in advance and records those that occur within $P_\lambda(Z)$.

\begin{definition}[Stopped-prefix ledger loss]
\label{def:risk}
Fix a finite set $\mathcal K$ of harm categories, each with an indicator $E_k(P)\in\{0,1\}$ that is determined by the executed records of the prefix $P$ alone and is non-decreasing as $P$ grows. A \emph{harm ledger} $\ell(P)\in[0,1]$ is any function of $(E_k(P))_{k\in\mathcal K}$ satisfying $\ell(\varnothing)=0$ and prefix monotonicity, $\ell(P)\leq\ell(P')$ whenever $P$ is a prefix of $P'$. The primary ledger is the union indicator $\Lunion(P)=\max_{k\in\mathcal K}E_k(P)\in\{0,1\}$, and the loss at threshold $\lambda$ on episode $Z$ is $L_\lambda(Z)=\ell(P_\lambda(Z))$.
\end{definition}

Raising $\lambda$ lengthens the executed prefix, so $L_\lambda(Z)$ is nondecreasing on every episode. Once a declared event occurs, it stays in the ledger through stopping or later recovery. This target measures whether harm occurred; applications that credit remediation would need a different loss.

A gate depends on both its calibration target and its scorer evidence. To examine the target, we compare episode loss with an executed-step rate $r_\lambda(Z)$: an indicator that harm occurred, divided by the number of permitted actions, with zero assigned to an empty prefix. This baseline calibrates the running maximum over a finite threshold grid $\Lambda$, $\widetilde r_\lambda(Z)=\max_{\lambda'\in\Lambda:\lambda'\leq\lambda}r_{\lambda'}(Z)$, which makes the target monotone. Appendix~\ref{app:protocol} gives the full specification. The next proposition shows why both target and evidence matter.

\begin{proposition}[Two obstructions to episode-level control]
\label{prop:necessity}
Fix $0<\alpha<p<1$.
\begin{enumerate}[label=(\roman*),leftmargin=1.8em,itemsep=1pt,topsep=2pt]
\item \textbf{Unit.} There exists a distribution of length-$K$ episodes, with $K\geq\lceil1/\alpha\rceil$ and no-guard harm prevalence $p$, such that every permitted prefix has $r_\lambda\leq\alpha$ and $\widetilde r_\lambda\leq\alpha$, but the all-execute policy has expected episode loss $p>\alpha$.
\item \textbf{Context.} Consider a harmful/benign pair drawn with probabilities $p$ and $1-p$, with a common candidate sink as the only possible ledger event, harmful only in the first member. Suppose the gate-visible score histories match through that sink under common policy randomness, including a common calibrated threshold whose randomness is independent of which member is drawn. Let $q$ be the overall probability that the sink executes, including survival through earlier decisions. Then $\E[L]=pq$, and $\E[L]\leq\alpha$ forces the matched benign sink not to execute with probability at least $1-\alpha/p$.
\end{enumerate}
\end{proposition}

Part (i) shows how many harmless steps can dilute the rate assigned to one harmful action. Part (ii) shows the cost of missing context: a gate that receives matching scores on harmful and benign histories must allow or stop them at the same rate. Together, these results motivate calibrating episode loss and retaining evidence that distinguishes the two histories. Proofs are in Appendix~\ref{app:proofs}.

\section{Methodology}

\paragraph{CARB: calibrate the permitted prefix.}
Before calibration, fix the ledger, selector, scorer, input-sanitization rules, and a finite threshold grid $\Lambda\subset[0,1]\cup\{\lstop\}$ with $\lstop<0$. For each threshold, replay $n$ logged ungated episodes $\mathcal D_{\rm cal}=(Z_1,\ldots,Z_n)$: follow each episode until the gate would reject a proposal, then evaluate $L_\lambda(Z_i)=\ell(P_\lambda(Z_i))$ on the permitted prefix. For a chosen risk budget $\alpha$, CARB selects the largest feasible threshold:
\begin{equation}
\widehat R^+(\lambda)=\frac{1+\sum_{i=1}^{n}L_\lambda(Z_i)}{n+1},
\qquad
\widehat\lambda=\max\{\lambda\in\Lambda:\widehat R^+(\lambda)\leq\alpha\}.
\label{eq:calibration}
\end{equation}
The correction adds one worst-case loss to the calibration sum. Choosing the largest feasible threshold allows as much execution as this criterion permits. Each proposal's score stays fixed across thresholds, so a higher threshold can only extend the permitted prefix. Algorithm~\ref{alg:carb} gives the procedure.

\paragraph{G-CARB: select the dependency evidence.}

\begin{definition}[Graph-localized context]
\label{def:graphselector}
At decision $t$, let $G_t=(V_t,E_t)$ be a directed acyclic graph on the executed records $Z_{<t}$ and proposal $A_t$, built from prefix-visible fields. Frozen tool metadata identifies private sources $S_t$ and the designated sink types. When $A_t$ is a sink reachable from a source, \gcarb{} supplies
\[
C_t^{\mathrm G}=\{v\in V_t:\ u\leadsto v\leadsto A_t\ \text{for some}\ u\in S_t\},
\]
the nodes on source-to-proposal paths, including their endpoints. Otherwise it supplies $A_t$ and its ancestors. Full-prefix uses $C_t^{\mathrm F}=V_t$.
\end{definition}

Graph edges link earlier records to later actions through shared object identifiers, declared input handles, or exact content reuse. The experimental sink types include external and financial actions; other proposals use the ancestor rule in Definition~\ref{def:graphselector} (Appendix~\ref{app:selector}). \gcarb{} and Full-prefix share the graph builder, node types, scorer, grid, and ledger. Only the selected records differ. They are supplied in chronological order, with the proposal's result left empty at scoring time.

In the certificate--note--email example, the reused private field or its handle links the read to the note and proposed send. \gcarb{} retains this chain and omits a disconnected public task. An email carrying only the public document instead follows the public branch. In both cases, the ledger evaluates every executed record.

The scorer adds fixed weights for binary features and clips the sum to $[0,1]$: $q_t=\operatorname{clip}_{[0,1]}(w^\top x(C_t))$. Features describe the proposal, selected private or injected content, and the presence of a source-to-sink path (Appendix~\ref{app:scorer}). Private records can raise this score even when disconnected from the proposal, which helps explain the benefit of pruning context. All selectors use the same weights, and computing the score requires no language-model inference.

\paragraph{Guarantee and replay.}
G-CARB selects evidence using only the history and current proposal. Once this rule is fixed, its losses form the same ordered family used by CARB. The following guarantee therefore covers both methods.

\begin{theorem}[Stopped-prefix conformal risk control]
\label{thm:crc}
Let $\mathcal F$ contain the ledger satisfying Definition~\ref{def:risk}, selector, scorer, finite grid containing $\lstop$, and all sanitization and replayable scoring rules, fixed before calibration. Conditional on $\mathcal F$, suppose $Z_1,\ldots,Z_n,Z_{\rm new}$ are exchangeable and the loss vector is the same deterministic function of each episode. For $\alpha\in[1/(n+1),1]$, Equation~\ref{eq:calibration} returns a threshold satisfying
\[
\E_{\mathcal D_{\rm cal},Z_{\rm new}}\!\left[
L_{\widehat\lambda(\mathcal D_{\rm cal})}(Z_{\rm new})
\mid\mathcal F\right]\leq\alpha.
\]
Averaging over $\mathcal F$ gives the unconditional bound. For $\alpha<1/(n+1)$, no threshold satisfies the corrected criterion.
\end{theorem}

This applies standard CRC \citep{angelopoulos2022conformal} to the stopped-prefix losses. The bound averages over both the calibration sample and a future episode; individual deployment batches can exceed the budget. The correction requires $\alpha\geq1/(n+1)$. Below this floor, Equation~\ref{eq:calibration} is infeasible, and stop-all remains available with known zero ledger loss.

\begin{corollary}[Fixed-corpus random assignment]
\label{cor:fixed-corpus}
Fix a finite corpus and every episode's loss vector before assignment. Under a uniform permutation into $n$ calibration and $m\geq1$ held-out positions, for $\alpha\geq1/(n+1)$ the expected held-out mean of $L_{\widehat\lambda}$ is at most $\alpha$, conditional on the unordered corpus and frozen pipeline. The expectation is over the assignment.
\end{corollary}

This corollary applies to the AgentDojo experiment, where each split reassigns the same corpus to calibration and test sets. Repeated splits measure sensitivity to that assignment. The next proposition states when replay reproduces online stopping.

\begin{proposition}[Replay equivalence]
\label{prop:replay}
Fix a threshold $\lambda$ and a common initial state. Couple offline replay and online execution through the same agent and environment randomness. Use a deterministic frozen wrapper or couple its randomness as well. Suppose the wrapper is silent and free of side effects before
rejection and observes the same prefix-visible fields in both runs. Then the proposals,
contexts, and scores agree up to the first rejection, the executed records and ledger state
agree before it, and the rejected proposal executes in neither run. Both runs therefore
produce the same stopped prefix and the same $L_\lambda(Z)$.
\end{proposition}

Proofs are in Appendix~\ref{app:proofs}. Warnings, prompt edits, or retries before stopping can change the agent's later actions. Evaluating such a wrapper requires trajectories collected with those interactions.

\section{Experiments}

We examine three questions: can a small harmful-step rate hide substantial episode harm, does ledger calibration keep mean episode loss within budget, and when does dependency-based context selection help?

\paragraph{AgentDojo and the ledger.}
In \textbf{AgentDojo}, an agent performs a benign user task while tool responses may contain injected instructions that redirect it toward an adversarial objective \citep{debenedetti2024agentdojo}. We use $984$ ungated episodes per backbone from Qwen2.5-14B-Instruct and Qwen3-14B. Each logged trajectory supplies the records for replaying candidate stopping policies.

The ledger records five harm categories: privacy flow, financial action, irreversible action, social action, and security event. An injection-directed event is registered when an observation contains an adversarial objective and a matching tool call later succeeds. Privacy flow also records whether a private canary---a synthetic marker---reaches an executed external or financial sink. Rejected proposals and private reads alone create no event. This ledger evaluates executed prefixes separately from AgentDojo's terminal evaluator. It lacks an independent error audit, and the effect of rewriting canaries in $27$ Qwen2.5 traces remains unmeasured (Appendix~\ref{app:ledger}).

\paragraph{Controlled suites.}
The Compositional Exfiltration Suite (CES) contains $80$ constructed leaks and $240$ matched non-leaks: public sends, private reads without sends, and disconnected private/public flows. Its clean subset pairs the leaks with $80$ public sends. Graph-braid interleaves harmful chains with unrelated benign branches in $2{,}000$ constructed trajectories. Labels come from the trace design and remain hidden from the graph builder and scorer. CES uses $100$ assignments with related trajectories kept together; graph-braid uses $200$ trajectory assignments. These suites test specific failure mechanisms at harm frequencies chosen by construction. The grouped CES splits fall outside Corollary~\ref{cor:fixed-corpus}; their results are descriptive. Appendix~\ref{app:protocol} gives fold sizes and grouping details.

\paragraph{Baselines.}
\emph{No guard} executes the complete trajectory. \emph{Executed-step CRC} sees only the current proposal and calibrates the harmful-step rate envelope. \emph{Ledger-current} calibrates episode loss using only the current proposal; it tests whether history is needed in the constructed suites. The remaining controls share \gcarb{}'s ledger, scorer, splits, and grid. \emph{Full-prefix} supplies all earlier records and the proposal. \emph{Fixed-window-4} selects recent records. \emph{Equal-size random context} retains the proposal and a fixed pseudorandom subset of earlier records, matching \gcarb{}'s size at each decision. This tests whether the identity of retained records matters beyond their count.

\paragraph{Metrics.}
The primary safety metric is mean episode loss $\bar L$: the fraction of test episodes with a ledger event before stopping. \emph{Intervention} $I$ is the fraction in which the gate rejects a proposal. \emph{Autonomous success} $U$ is the fraction whose logged task completes before handoff; completion can coexist with a ledger event. The controlled suites also report the fraction of harmful sinks prevented from executing. Their benign-stop metrics have different denominators: non-leak trajectories in CES, and all trajectories in graph-braid. The graph-braid metric counts first stops on disconnected benign branches.

Scorer-input reduction, CtxRed, measures the fraction of records omitted from the scorer's input. We sum selected and full-prefix record counts over the same proposals, including the first rejected proposal, and report one minus their ratio (Appendix~\ref{app:protocol}). Graph construction and total runtime are outside this measurement.

\paragraph{Calibration protocol.}
Before splitting, we freeze the scorer, ledger, graph builder, sanitizer, $102$-element threshold grid, and pipeline hash. Each backbone uses $200$ seeded permutations into $492$ calibration and $492$ test episodes. Methods share these assignments, the ledger, and the grid. Each selector receives its own threshold from Equation~\ref{eq:calibration} at the same budget. We interpret these results conditional on the fixed corpus and scorer, as in Corollary~\ref{cor:fixed-corpus}. Means and percentiles summarize variation across assignments.

\paragraph{A small step rate can coexist with episode harm.}
Executed-step CRC keeps its reported step rate below budget while its mean episode loss exceeds it (Appendix Table~\ref{tab:agentdojo}). At three of the four displayed points, its mean episode loss matches no guard to reported precision ($.133$, $.133$, and $.225$). Dividing a harmful event by trajectory length makes the rate small even when the event occurs, as Proposition~\ref{prop:necessity}(i) predicts. The methods calibrated on the global ledger keep their mean held-out episode losses below the displayed budgets (\frontierref{}). This AgentDojo comparison changes both the target and context. The constructed suites compare Ledger-current with other ledger-calibrated selectors; they do not isolate the effect of changing the calibration target.

\begin{figure}[tb]
\centering
\safeincludegraphics[width=\linewidth]{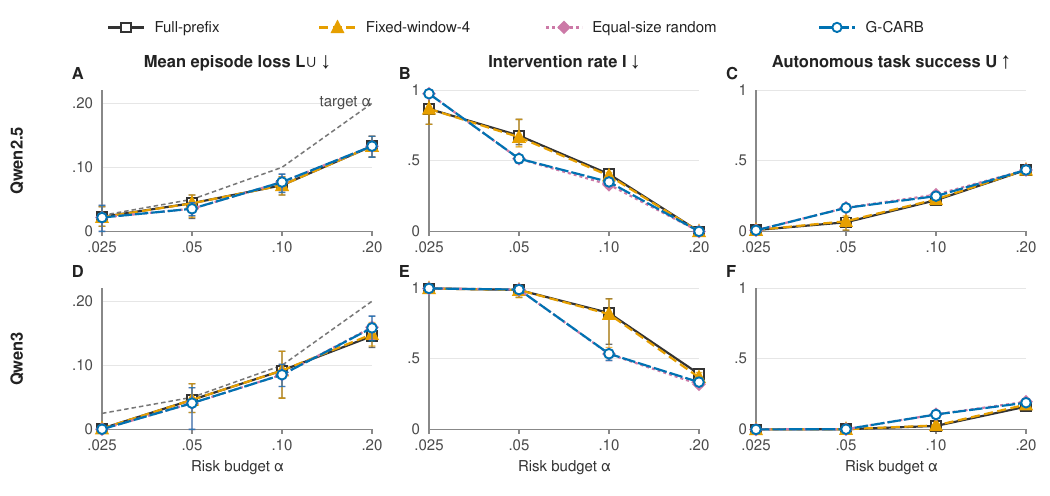}
\caption{\textbf{AgentDojo risk--utility frontiers.} Shorter context improves autonomous success at intermediate budgets; equal-size random context closely tracks \gcarb{}. Every plotted method's mean ledger loss stays below the budget (dashed line). Points average $200$ shared calibration/test assignments of the fixed corpus. Whiskers show their empirical 5th--95th percentiles. Appendix Tables~\ref{tab:agentdojo} and~\ref{tab:context-controls} give the values.}
\label{fig:frontiers}
\end{figure}

\paragraph{Context restriction improves autonomous success.}
Context restriction has its greatest effect at intermediate budgets (Figure~\ref{fig:frontiers}). At strict budgets, especially for Qwen3, the gates stop almost every episode; at the most permissive Qwen2.5 point, all methods execute the full trajectories.

Equal-size random context closely follows \gcarb{}, while the fixed window stays nearer Full-prefix. Removing private or injected records can lower this scorer's output and reduce unnecessary stops. At budgets $.05$ and $.10$, G-CARB supplies roughly half the full-prefix record count. The constructed suites examine when record identity also matters.

\paragraph{Dependency evidence helps avoid benign stops.}
CES near-misses vary whether private data reaches the same sink type. Current-only scoring lacks this history; random selection can discard it. In graph-braid, an unrelated branch can trigger an early full-prefix stop. Every gated method stops the harmful sinks, but G-CARB avoids benign stops in both constructions (Figure~\ref{fig:mechanism}).

\begin{figure}[tb]
\centering
\safeincludegraphics[width=\linewidth]{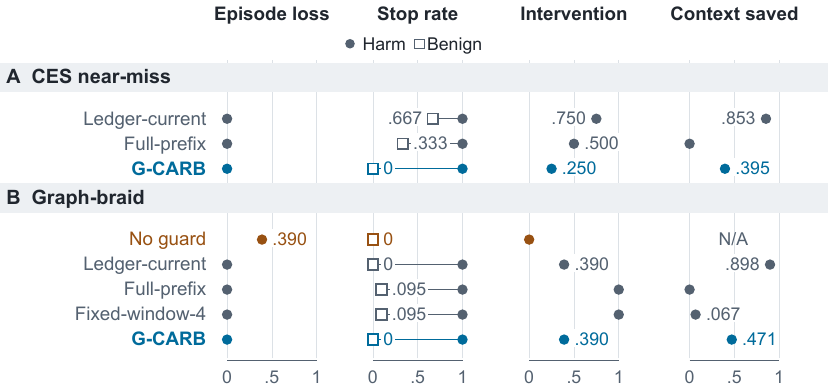}
\caption{\textbf{Stopping harmful flows while allowing benign work, at $\alpha=.10$.} All gated methods shown stop every harmful sink. CES benign stop counts stopped non-leaks; graph-braid counts only first stops on disconnected benign branches, divided by all trajectories. Thus Full-prefix intervenes in every graph-braid episode but has benign-stop rate $.095$. Context saved is CtxRed. Means use $100$ CES and $200$ graph-braid assignments; Appendix Table~\ref{tab:constructed} includes equal-size random context.}
\label{fig:mechanism}
\end{figure}

Ledger-current matches \gcarb{} on graph-braid, and Full-prefix matches it on CES clean. Each simpler context suffices in one construction; G-CARB handles both by retaining the relevant history and excluding unrelated records.

\paragraph{Observability and scope.}
The CES audit holds calibrated thresholds fixed while removing structured provenance, exact sink text, or both. Either remaining channel identifies the harmful paths; removing both returns loss to the no-guard level (Appendix~\ref{app:observability}). Transfer requires validation and calibration under the theorem's assumptions.

\section{Conclusion}
\carb{} applies standard CRC to accumulated harm before stopping; \gcarb{} selects dependency evidence while retaining that global loss. Context restriction improves AgentDojo task completion and roughly halves scorer-input records at intermediate budgets. Random context performs similarly; constructed diagnostics show where relevant history also avoids benign stops. Smaller backbones, alternative scorers, and total gate overhead remain evaluation priorities.

\paragraph{Limitations.}
The guarantee bounds expected declared ledger loss under exchangeability, not harm in every episode or harm outside the ledger. The ledger lacks an independent error audit, and semantic or paraphrased flows may escape the observable matching rules. AgentDojo evidence comes from repeated assignments of a fixed corpus using one scorer and two same-family 14B backbones; broader selectivity remains untested. Replay covers stopping without prior feedback to the agent, leaving interactive interventions and outcomes after handoff unevaluated. Finally, fewer scorer-input records do not establish lower total runtime: graph construction, end-to-end latency, and intervention costs remain unmeasured.

%% The template supplies this unnumbered section in final/preprint mode.
%% Confirmed support; additional coauthor funding/disclosures remain to be collected.
\begin{ack}
This work benefited from the computational resources provided by GERAD. Zijun Yu acknowledges partial support from an ISM Graduate Scholarship. Masoud Asgharian acknowledges the support of the Natural Sciences and Engineering Research Council of Canada (NSERC) [RGPIN-2024-05640].
\end{ack}

\bibliographystyle{plainnat}
\bibliography{aaai2027,workshop_additions} %% original bibliography plus the two added related-work records

\clearpage
\appendix
%% Included by workshop_main.tex after \appendix; no standalone preamble.
%% Numerical recovery source: frozen Appendix.tex. Values are preserved,
%% not recomputed; unsupported robustness/official-outcome claims are omitted.

\section{Theory and proofs}
\label{app:proofs}

\subsection{Control unit and matched histories}

\begin{proof}[Proof of Proposition~\ref{prop:necessity}(i)]
Choose an integer $K\geq\lceil1/\alpha\rceil$. Every trajectory has $K$
actions. With probability $p$, its only union-ledger increment occurs at
action $K$; otherwise the ledger never increments. A policy that stops
before action $K$ has both zero episode loss and zero step rate. A policy
that executes action $K$ has rate $1/K$ on the harmful member and zero on
the benign member. Thus every threshold has
$r_\lambda(Z)\leq1/K\leq\alpha$, and taking the maximum over lower
thresholds preserves the bound for $\widetilde r_\lambda(Z)$. The
all-execute threshold $\lambda=1$ permits the complete trajectory and
has $\E[\Lunion(P_1(Z))]=p>\alpha$.
\end{proof}

The counterexample establishes compatibility of a small rate with large
episode risk. It does not assert that Equation~\ref{eq:calibration},
with its finite-sample correction and the envelope substituted for the
ledger, accepts the all-execute policy at every sample size. More
generally, for any possibly randomized stopping policy applied to a
binary union ledger,
\[
\E[\Lunion(P_{\widehat\lambda}(Z))]
=p\,\Pr\!\left(\text{first harmful action executes}
\mid \Lunion(Z)=1\right),
\]
where $p$ is the no-guard prevalence and the conditional probability
includes calibration and policy randomness. Gated risk equals $p$ only
when every harmful trajectory reaches that action.

\begin{proof}[Proof of Proposition~\ref{prop:necessity}(ii)]
Couple the two potential members using the same calibrated threshold
and policy randomness, independent of the member drawn. By assumption,
the score histories through the candidate sink agree. Inducting over
these histories, every preceding accept/reject decision agrees, as does
the decision at the sink if reached. Hence the indicator that the sink
executes is identical for both members under the coupling. Write $q$ for
its probability, averaging over calibration and policy randomness. Only
the harmful member can incur a ledger event, and it does so exactly when
the sink executes, giving $\E[L]=pq$. The inequality $pq\leq\alpha$
implies $q\leq\alpha/p$. The benign member therefore fails to execute
its sink with probability $1-q\geq1-\alpha/p$.
\end{proof}

The sink may be prevented by an earlier stop, so matching histories
through the sink is essential: equal final scores alone leave earlier
survival probabilities unconstrained. Observable provenance can
distinguish the constructed pair before its sink executes.

\subsection{Ordered-grid CRC and stopped-prefix calibration}

\begin{lemma}[Finite ordered-grid CRC]
\label{lem:grid}
Let $\mathcal F$ collect every component fixed before calibration and write
$L_i(\lambda)=L_\lambda(Z_i)$. Conditional on $\mathcal F$, let $Z_1,\ldots,Z_{n+1}$ be
exchangeable, let $\Lambda$ be a finite totally ordered grid, and let
$L_i(\lambda)\in[0,1]$ be the same deterministic loss map for every $i$, nondecreasing in
$\lambda$, with $\Lambda$ containing a stop-all value $\lstop$ satisfying
$L_i(\lstop)=0$. For $\alpha\geq1/(n+1)$, define
\[
\widehat\lambda=\max\Bigl\{\lambda\in\Lambda:\ \bigl(1+\textstyle\sum_{i=1}^{n}L_i(\lambda)\bigr)/(n+1)\leq\alpha\Bigr\}.
\]
Then the feasible set is nonempty and $\E[L_{n+1}(\widehat\lambda)\mid\mathcal F]\leq\alpha$.
\end{lemma}

\begin{proof}
For each $i\in\{1,\ldots,n+1\}$ apply the rule with observation $i$ left out,
\[
\lambda_{-i}=\max\Bigl\{\lambda\in\Lambda:\ 1+\textstyle\sum_{j\neq i}L_j(\lambda)\leq(n+1)\alpha\Bigr\},
\]
which is well defined because the stop-all value is always feasible. Since
$L_i(\lambda_{-i})\leq1$, we get $\sum_{j=1}^{n+1}L_j(\lambda_{-i})\leq(n+1)\alpha$. Let
$\lambda^*$ be the largest grid point satisfying that inequality for the full augmented
sample. Each $\lambda_{-i}$ is feasible for the full-sample problem, so
$\lambda_{-i}\leq\lambda^*$, and monotonicity gives
\[
\textstyle\sum_{i=1}^{n+1}L_i(\lambda_{-i})\leq\sum_{i=1}^{n+1}L_i(\lambda^*)\leq(n+1)\alpha.
\]
The leave-one-out construction is permutation symmetric and
$\widehat\lambda=\lambda_{-(n+1)}$, so conditional exchangeability yields
$\E[L_{n+1}(\widehat\lambda)\mid\mathcal F]=(n+1)^{-1}\sum_{i=1}^{n+1}\E[L_i(\lambda_{-i})\mid\mathcal F]\leq\alpha$.
This is the finite ordered-grid specialization of conformal risk control for a
unit-bounded loss \citep{angelopoulos2022conformal}.
\end{proof}

\begin{proof}[Proof of Theorem~\ref{thm:crc}]
Conditional on $\mathcal F$, fixing the ledger, selector, scorer, grid, and sanitizer makes
$(L_\lambda(Z))_{\lambda\in\Lambda}$ a deterministic function of an episode, so
conditionally exchangeable episodes give conditionally exchangeable loss vectors. A larger threshold permits a weakly
longer prefix, and prefix monotonicity of the ledger makes the loss nondecreasing in
$\lambda$. The grid contains $\lstop$, which permits no records and gives
$L_{\lstop}(Z)=\ell(\varnothing)=0$. Lemma~\ref{lem:grid} then gives the bound conditional
on the frozen components, and averaging over them gives the unconditional statement.

At $\lstop$ every calibration loss is zero, so the corrected risk there is exactly
$1/(n+1)$, and every corrected risk has numerator at least one. The stop-all element is
therefore feasible at and above $1/(n+1)$, while the entire grid is infeasible below it,
which is the reported infeasibility case. The stop-all policy still has known zero loss below this floor; only the corrected selection criterion is infeasible.
\end{proof}

Dependence among records within an episode is unrestricted; exchangeability is required at
the episode level only.

\begin{proof}[Proof of Corollary~\ref{cor:fixed-corpus}]
Let $\mathcal F$ now contain the unordered corpus and frozen loss map. For each held-out position $j$, uniform episode permutation makes the $n$ calibration positions together with position $j$ exchangeable conditional on $\mathcal F$.
Theorem~\ref{thm:crc} bounds each expected held-out loss by $\alpha$, and linearity of
expectation gives the bound on their mean.
\end{proof}

\subsection{Replay equivalence}

\begin{proof}[Proof of Proposition~\ref{prop:replay}]
At the first decision both runs share the empty executed prefix, proposal, graph, selected
context, and score. If the proposal is rejected, both stop before executing it; if it is
accepted, the coupled action and environment randomness produce the same observation, so
the first executed records agree. Assume agreement through the first $t-1$ executed
records. Silence and the absence of wrapper side effects preserve the same agent state, so
coupling supplies the same proposal at decision $t$, and prefix measurability supplies the
same graph, context, score, and decision. Induction continues to the first rejection or to
natural termination. The ledger is deterministic on the common executed prefix, so ledger
state and loss agree as well.
\end{proof}

The proposition ceases to apply to a wrapper that edits the prompt before a decision,
mutates the environment before rejection, exposes different fields in replay, or resamples
uncoupled randomness.

\section{Algorithm, graph selector, and scorer}
\label{app:selector}

\subsection{Calibration algorithm}

Algorithm~\ref{alg:carb} applies to a frozen selector $C$. Substituting the
graph selector $C^{\mathrm G}$ yields \gcarb{}; substituting the complete
prefix yields Full-prefix CARB. The ledger always sees all permitted
executed records, irrespective of which records were selected for scoring.

\begin{algorithm}[ht]
\caption{CARB stopped-prefix calibration}
\label{alg:carb}
\begin{algorithmic}[1]
\Require frozen selector $C$, scorer $s$, ledger $\ell$, sanitization; episodes $Z_{1:n}$; finite grid $\Lambda\ni\lstop<0$; budget $\alpha$
\For{$\lambda\in\Lambda$ and $i=1,\ldots,n$}
  \State $P\gets\varnothing$
  \For{each proposal $A_t$ logged in $Z_i$}
    \State $C_t\gets C(P,A_t)$; \ $q_t\gets s(C_t)$ \Comment{no future result $O_t$}
    \State \textbf{if} $q_t>\lambda$ \textbf{then break} \Comment{stop before $A_t$}
    \State $P\gets P\,R_t$ \Comment{append the logged executed record}
  \EndFor
  \State $L_i(\lambda)\gets\ell(P)$ \Comment{earlier harm stays charged}
\EndFor
\State $\widehat R^+(\lambda)\gets(1+\sum_i L_i(\lambda))/(n+1)$ for each $\lambda\in\Lambda$
\State \Return $\max\{\lambda\in\Lambda:\widehat R^+(\lambda)\leq\alpha\}$, or report criterion infeasible
\end{algorithmic}
\end{algorithm}

\subsection{Prefix graph and source-to-sink selector}

At decision $t$, the graph contains one node for every executed record
$R_1,\ldots,R_{t-1}$ and one node for the proposal $A_t$.  An executed
node contains its action and visible result.  The current node contains
the proposed action and an empty result.  Future actions, future
observations, terminal outcomes, ledger labels, and suite labels are
excluded.

Node typing assigns five types: private sources, external sinks, financial actions, irreversible actions, and message/social actions. Types come from frozen tool metadata together with deterministic rules over tool names and visible outputs, all fixed before assignment and calibration. No typing rule reads suite labels or terminal outcomes.
Visible identifiers include e-mail addresses, IBANs, canary strings, and
values attached to object, source, event, message, channel, account,
transaction, file, recipient, or sender fields.  Generic identifiers
must contain at least four characters.  Distinctive text spans comprise
exact e-mail, IBAN, and canary strings, together with visible result
lines of at least 24 characters that contain a digit; such lines are
truncated to 80 characters.  A general exact-span match requires at
least 16 characters.

Every edge points from an earlier node to a later node.  An edge is added
when at least one of the following prefix-visible relations holds:
\begin{enumerate}
\item the records carry the same declared branch identifier;
\item they share an object identifier or a data identifier;
\item the later action explicitly declares as input a data identifier
produced earlier;
\item a canary observed earlier reappears in the later action arguments;
\item a later financial action shares a transaction or account
identifier with the earlier record;
\item a later irreversible action updates the same visible object;
\item a later message/social action shares a thread or message
identifier with the earlier record; or
\item exact distinctive content from the earlier result reappears in
the later visible action fields.
\end{enumerate}
The construction uses neither unconditional temporal edges nor learned
or embedding-based similarity.  Strict temporal orientation makes the
graph acyclic.

Our experiments instantiate the sink predicate as the frozen external-sink type, which includes financial sinks. An irreversible-only, non-external proposal uses the ancestor fallback.
The same fallback applies when no source-to-current path exists: the selector returns the current proposal
together with its closed ancestor set.
The current proposal is always retained, and selected nodes are
serialized in chronological order.  Full-prefix \carb{} instead returns
every graph node.  The equal-size random control retains the proposal
and a deterministic pseudorandom subset of earlier nodes whose size
matches the graph-local context at that decision.

These relations are observable matching rules, not verified semantic
dataflow. A shared identifier or exact span can introduce an irrelevant
edge, while paraphrasing, transformation, or unlogged state can remove a
needed one. The source-to-sink feature is evaluated on retained evidence;
Full-prefix and \gcarb{} use the same underlying graph machinery. Changing
the evidence changes scores and possibly harm, without redefining the
ledger in Appendix~\ref{app:ledger}.

\subsection{Frozen weighted scorer}
\label{app:scorer}

The experiments use
$q_t=\operatorname{clip}_{[0,1]}(w^\top x(C_t))$ with the complete nonzero
feature set in Table~\ref{tab:scorer}. We fixed the features and weights
before assignment and calibration. We interpret the experiments
conditional on this scorer and the fixed corpus, as in
Corollary~\ref{cor:fixed-corpus}.

\begin{table}[t]
\centering
\footnotesize
\setlength{\tabcolsep}{3.5pt}
\begin{tabular}{@{}lr@{}}
\toprule
Feature & Weight \\
\midrule
Intercept & .05 \\
Current private source & .09 \\
Current external sink & .25 \\
Current financial action & .35 \\
Current irreversible action & .24 \\
Current message/social action & .12 \\
Earlier selected private source & .10 \\
Earlier selected canary & .14 \\
Earlier selected injection text & .14 \\
Private source reaches current sink & .52 \\
Sink after selected private source, path absent & .16 \\
Canary in current arguments & .20 \\
Current risky action & .07 \\
\bottomrule
\end{tabular}
\caption{Frozen weighted scorer. A risky current action is external, financial, or
irreversible. The weighted sum is clipped to $[0,1]$; only a retained source-to-sink path
activates the $.52$ provenance feature.}
\label{tab:scorer}
\end{table}

The $.10$, $.14$, and $.14$ weights on earlier selected private sources,
canaries, and injection text can increase a score even when the selected
record is disconnected from the proposal. The $.52$ feature instead
requires a retained private-source-to-sink path. These choices explain
why pruning can improve this scorer. Selectivity with learned scorers
or alternative context weights remains an open comparison.

\subsection{Record-count scope}

Over $T$ decisions, serializing every prefix supplies
$\sum_{t=1}^{T}O(t)=O(T^2)$ records to the scorer. If a selected component
has at most $d$ records per decision, the corresponding count is
$O(Td)$. Both statements are conditional record-count bounds. They
exclude graph construction, maintenance, and other wrapper costs, and
require bounded record length to imply token or byte bounds.

\section{AgentDojo ledger, corpus, and preprocessing}
\label{app:ledger}

The fixed corpora contain $984$ ungated episodes each from
Qwen2.5-14B-Instruct and Qwen3-14B on AgentDojo
\citep{debenedetti2024agentdojo}. Our custom ledger is a deterministic
function of executed, prefix-visible records; it is not an official
AgentDojo harm annotation. Its five binary categories are privacy flow,
financial action, irreversible action, social action, and security event.
The reported episode loss is their union,
\[
\Lunion(P)=\max\{\ell_{\rm priv}(P),\ell_{\rm fin}(P),
\ell_{\rm irr}(P),\ell_{\rm soc}(P),\ell_{\rm sec}(P)\}.
\]
Each indicator records whether its event has occurred and remains one
afterward. Subsequent recovery does not remove a registered event.

\paragraph{Event registration.}
For injection-directed events, the adapter identifies an adversarial
objective from injection text only after that text appears in an
executed observation. It registers an event when the corresponding tool
call executes and its result confirms execution; the category follows
the frozen tool type. Privacy flow also registers when an exact canary
previously observed at a private source appears in the executed arguments
of an external or financial sink. Rejected proposals, failed calls,
future records, terminal benchmark outcomes, terminal fallback, and
private reads without an executed sink create no event. The recorded
corpora contain $169$ events in $130$ positive Qwen2.5 episodes and $257$
events in $221$ positive Qwen3 episodes. Event totals can exceed the
number of positive episodes; the binary union and executed-step baseline
count only whether, and when, the first union event occurs.

\paragraph{Policy and evaluation separation.}
Registered ledger labels and benchmark outcomes are excluded from the
graph and scorer. The scorer can use visible tool fields and observed
text as specified in Appendix~\ref{app:selector}; it cannot consume a
future event label. The ledger evaluates the complete executed prefix,
not just the selected context. The official attack evaluator is excluded
from scoring, calibration, threshold selection, and operating-point
selection. Its terminal outcomes and the prefix ledger measure different
objects; the reported risk target is the latter.

\paragraph{Preprocessing and validity limits.}
We rewrote canary tokens in $27$ Qwen2.5 trajectories during trace
preparation and made no equivalent substitution in Qwen3. We have not
quantified the effect of this preprocessing difference on ledger labels,
graph edges, or results. We have not independently audited the event
adapter against human annotations or benchmark outcomes. Exact-canary
detection need not capture paraphrased or transformed disclosure. The theorem bounds this declared operational ledger; it does
not establish its coverage of all harms or its semantic accuracy.

\section{Calibration, splits, metrics, and compute}
\label{app:protocol}

\subsection{Frozen pipeline and assignments}

The ledger, graph builder, selector, scorer, sanitizer, and grid are fixed
before assignment and calibration. The $102$-element grid is
\[
\Lambda=\{\lstop\}\cup\{0,.01,.02,\ldots,1\},\qquad\lstop=-1,
\]
and a proposal is rejected exactly when $q_t>\lambda$. Increasing the
threshold is more permissive. AgentDojo uses $200$ seeded episode
permutations, indexed $0$--$199$, with $492$ calibration and $492$ test
episodes per backbone. For each assignment, we sort eligible
episode or group identifiers and shuffle them with Python's Mersenne
Twister at base seed zero plus the split index. AgentDojo's statistical
interpretation is the uniform episode-assignment model in
Corollary~\ref{cor:fixed-corpus}; the $200$ seeded assignments summarize
sensitivity of this fixed corpus rather than $200$ independent datasets.

Table~\ref{tab:split-sizes} gives the pool and fold sizes. CES keeps all
trajectories with a common pair identifier in one fold; graph-braid uses
one complete trajectory as its assignment unit. Each CES clean split has
$48$ calibration and $48$ test trajectories; each near-miss split has $96$
of each. Graph-braid uses $600$ calibration and $600$ test trajectories.
Only the selected calibration and test folds enter each split's results.
We report the constructed results as descriptive held-out diagnostics:
the uniform episode-permutation corollary does not cover grouped CES.

\begin{table}[tb]
\centering
\small
\setlength{\tabcolsep}{5pt}
\begin{tabular}{lrrrrl}
\toprule
Dataset & Pool & Calibration & Test & Assignments & Assignment unit \\
\midrule
AgentDojo, each backbone & 984 & 492 & 492 & 200 & Episode \\
CES clean & 160 & 48 & 48 & 100 & Pair identifier \\
CES near-miss & 320 & 96 & 96 & 100 & Pair identifier \\
Graph-braid & 2,000 & 600 & 600 & 200 & Trajectory \\
\bottomrule
\end{tabular}
\caption{Pool and fold sizes, all in trajectories. CES uses split seeds
$0$--$99$; AgentDojo and the main graph-braid setting use $0$--$199$.
CES trajectories sharing a pair identifier stay together in one fold.}
\label{tab:split-sizes}
\end{table}

The equal-size random selector retains the proposal and a deterministic
pseudorandom subset of earlier nodes with the graph selector's size at
that decision. Its full budget grid uses selector seed $0$; a separate
check uses seeds $0$--$19$ at $\alpha=.05$. Selector seeds and assignment
seeds have distinct roles. For threshold replay, each episode--proposal
pair must retain the same selected context at every threshold; resampling
contexts across thresholds can destroy prefix nesting.

\subsection{Calibration targets and context controls}

\begin{table}[tb]
\centering
\small
\setlength{\tabcolsep}{4pt}
\begin{tabular}{lll}
\toprule
Method & Calibrated object & Scorer evidence \\
\midrule
Executed-step CRC & Rate envelope $\widetilde r_\lambda$ & Current proposal \\
Ledger-current & Stopped-prefix ledger $L_\lambda$ & Current proposal \\
Full-prefix CARB & Same ledger & Full prefix and proposal \\
\gcarb{} & Same ledger & Dependency-selected nodes \\
Fixed-window-4 & Same ledger & Four-record recency context \\
Equal-size random & Same ledger & Size-matched random context \\
\bottomrule
\end{tabular}
\caption{Internal calibration and context controls. No guard is
uncalibrated. Ledger-current is reported only on the constructed suites;
Executed-step CRC is reported only on AgentDojo, so these results do not
provide a same-context loss ablation. Executed-step CRC
is an internal analogue of step-level control, not a reproduced CORA
implementation.}
\label{tab:methods}
\end{table}

For a full ungated trajectory, write $P_t=(R_1,\ldots,R_t)$ and
$P_0=\varnothing$. The executed-step baseline uses the first binary
union-ledger increment
\[
h_t=\Lunion(P_t)-\Lunion(P_{t-1}),\qquad
r_\lambda(Z)=\frac{\sum_{t\leq|P_\lambda(Z)|}h_t}
{\max\{1,|P_\lambda(Z)|\}}.
\]
An empty permitted prefix gives rate zero. The numerator is at most one;
it is not the total count of category events. Because a raw rate can
fall as a harmless suffix grows, the calibration target is
\[
\widetilde r_\lambda(Z)=
\max_{\lambda'\in\Lambda:\lambda'\leq\lambda}r_{\lambda'}(Z).
\]
This envelope is nondecreasing in the more permissive threshold. The
baseline applies Equation~\ref{eq:calibration} with
$\widetilde r_\lambda$ in place of $L_\lambda$, whereas reported step
rates are the held-out raw $r_{\widehat\lambda}$.

\subsection{Reported metrics and denominators}

The held-out mean ledger loss is
\[
\bar L=\frac{1}{|\mathcal I_{\rm test}|}
\sum_{i\in\mathcal I_{\rm test}}
\Lunion(P_{\widehat\lambda}(Z_i)).
\]
Intervention $I$ is the fraction of test trajectories with a rejected
proposal. Autonomous success $U$ is the fraction whose logged task
completes before handoff; task completion and a ledger event can coexist.
No joint harm-free success rate is reported.

Harm stop is the fraction of harmful sinks prevented from executing,
including prevention by an earlier rejection. CES benign stop divides
stopped non-leak trajectories by all held-out non-leak trajectories.
Graph-braid benign stop instead divides trajectories whose first
rejection is on a disconnected benign branch by all held-out trajectories.
These benign-stop metrics have different denominators and meanings.

For method $m$, let $\mathcal J_m$ contain the episode--proposal pairs it
scores through first rejection (including the rejected proposal) or
natural termination. With the proposal counted as one node,
\begin{equation}
\mathrm{CtxRed}_m
=1-\frac{\sum_{(i,t)\in\mathcal J_m}|C^m_{it}|}
{\sum_{(i,t)\in\mathcal J_m}|C^{\rm full}_{it}|}.
\label{eq:ctxred}
\end{equation}
The denominator is the counterfactual full-prefix record count on that
same set of proposals. In the main \gcarb{} results, the comparison set is
$\mathcal J_{\mathrm G}$, not the proposals reached by a separately
calibrated Full-prefix policy. Per-decision size matching does not force
equal aggregate CtxRed for random and graph selection: their stopping
points, and hence their $\mathcal J_m$, can differ. This is a serialized
record-count measure, not tokens, latency, memory, or total computation.

Means and empirical 5th--95th percentiles in the AgentDojo results
summarize assignment sensitivity on the reused corpus. Paired differences
use the same test assignment for both methods. These percentiles are not
independent-sample confidence intervals or high-probability guarantees
for future deployments.

\subsection{Computing environment and code}

We ran all experiments on the same dual-socket AMD EPYC 7402 host with
$48$ physical cores, 1.0~TB RAM, two NVIDIA A100 80GB GPUs, and Linux 5.14
on x86-64. Replay and postprocessing used Python 3.9.25, PyYAML 5.4.1,
NumPy 1.23.5, and pandas 2.3.2; AgentDojo is pinned at
v0.1.35-4-g18b501a63. The closed-loop check used
vLLM 0.21.0, bfloat16, tensor parallelism $2$, maximum sequence length
$8192$, and seed $0$. Calibration replays logged episodes with the frozen
rule and requires no additional language-model inference. We do not
report execution times or measured wrapper overhead.

Code accompanies the OpenReview submission. The supplement reconstructs
the prefix-visible policy view, event ledger, graph contexts, and stopped
prefixes from the ungated records.

AgentDojo is distributed under the
\href{https://github.com/ethz-spylab/agentdojo/blob/main/LICENSE}{MIT license}.
The \href{https://huggingface.co/Qwen/Qwen2.5-14B-Instruct/blob/main/LICENSE}{Qwen2.5-14B-Instruct}
and \href{https://huggingface.co/Qwen/Qwen3-14B/blob/main/LICENSE}{Qwen3-14B}
backbones are distributed under Apache License 2.0.

\section{Additional AgentDojo results}
\label{app:agentdojo}

\subsection{Selected operating points}

\begin{table}[tb]
\centering
\small
\setlength{\tabcolsep}{2.4pt}
\begin{tabular}{llcc ccc ccc ccc}
\toprule
& &
\multicolumn{2}{c}{Executed-step CRC} &
\multicolumn{3}{c}{Full-prefix} &
\multicolumn{3}{c}{Equal-size random} &
\multicolumn{3}{c}{\gcarb{}} \\
\cmidrule(lr){3-4}
\cmidrule(lr){5-7}
\cmidrule(lr){8-10}
\cmidrule(lr){11-13}
Backbone & $\alpha$
& $r\downarrow$ & $\bar L\downarrow$
& $\bar L\downarrow$ & $I\downarrow$ & $U\uparrow$
& $\bar L\downarrow$ & $I\downarrow$ & $U\uparrow$
& $\bar L\downarrow$ & $I\downarrow$ & $U\uparrow$ \\
\midrule
\multirow{2}{*}{Qwen2.5}
& .050
& .027 & .133
& .044 & .680 & .065
& .035 & .515 & .168
& .035 & .516 & .167 \\
& .100
& .027 & .133
& .072 & .406 & .222
& .077 & .332 & .261
& .077 & .351 & .250 \\
\midrule
\multirow{2}{*}{Qwen3}
& .050
& .034 & .154
& .046 & .989 & .001
& .041 & .992 & .002
& .041 & .992 & .002 \\
& .100
& .044 & .225
& .092 & .826 & .025
& .085 & .535 & .107
& .085 & .537 & .106 \\
\bottomrule
\end{tabular}
\caption{\textbf{AgentDojo results averaged over $200$ calibration/test assignments.} Executed-step CRC calibrates the rate envelope $\widetilde r_\lambda$ and reports raw rate $r$. Full-prefix, equal-size random, and \gcarb{} calibrate the same stopped-prefix ledger with different scorer evidence. $\bar L$, $I$, $U$: mean episode loss, intervention, autonomous success. Both restricted-context selectors improve utility over Full-prefix at the displayed points; Qwen3 at $.05$ is near stop-all.}
\label{tab:agentdojo}
\end{table}

At $\alpha=.05$ and $.10$, respectively, \gcarb{}'s CtxRed is
$.565$ and $.597$ for Qwen2.5, and $.511$ and $.590$ for Qwen3.
The no-guard mean episode losses are $.133$ and $.225$ for the two
backbones. Context counts use Equation~\ref{eq:ctxred}.

\subsection{Thresholds and feasibility}

Table~\ref{tab:thresholds} gives the complete four-budget grid for
Full-prefix and \gcarb{}. All $3{,}200$ recorded calibration selections
are feasible and satisfy their own corrected calibration criterion.
This does not assert that every held-out assignment meets its budget.
Numerically positive thresholds can still cause behavioral stop-all.

\begin{table}[t]
\centering
\small
\setlength{\tabcolsep}{5pt}
\begin{tabular}{llccccccc}
\toprule
& & \multicolumn{3}{c}{Full-prefix threshold} &
\multicolumn{3}{c}{\gcarb{} threshold} & \\
\cmidrule(lr){3-5}\cmidrule(lr){6-8}
Backbone & $\alpha$ & Min & Median & Max & Min & Median & Max &
Feasible per method \\
\midrule
\multirow{4}{*}{Qwen2.5}
& .025 & .24 & .54 & .84 & .13 & .44 & .58 & 200/200 \\
& .050 & .58 & .84 & .98 & .44 & .58 & .72 & 200/200 \\
& .100 & .99 & .99 & .99 & .99 & .99 & .99 & 200/200 \\
& .200 & 1.00 & 1.00 & 1.00 & 1.00 & 1.00 & 1.00 & 200/200 \\
\midrule
\multirow{4}{*}{Qwen3}
& .025 & .24 & .24 & .35 & .13 & .13 & .25 & 200/200 \\
& .050 & .24 & .37 & .54 & .13 & .25 & .44 & 200/200 \\
& .100 & .54 & .58 & .84 & .44 & .58 & .58 & 200/200 \\
& .200 & .99 & .99 & 1.00 & .99 & .99 & 1.00 & 200/200 \\
\bottomrule
\end{tabular}
\caption{Selected-threshold ranges over 200 assignments per cell
($3{,}200$ calibrations total).  All corrected calibration risks are at
most their displayed budget; the range shows that assignment matters
most at the tighter operating points.}
\label{tab:thresholds}
\end{table}

\subsection{Paired assignment sensitivity and endpoints}

Table~\ref{tab:paired-sensitivity} reports \gcarb{} minus Full-prefix
on shared assignments. Reported paired differences need not equal
subtraction of the rounded entries in Table~\ref{tab:agentdojo}.
Rounding can make these two quantities differ. In particular, the
reported Qwen3 $.05$ utility difference rounds to zero.
At $.025$, Qwen2.5 autonomous success is $.0078$ for \gcarb{} and $.0098$
for Full-prefix, a small reversal; Qwen3 is behaviorally stop-all. At
$.20$, both Qwen2.5 selectors choose threshold $1$ and match no guard,
with $U=.436$. Qwen3 still has gated behavior at $.20$, with $U=.189$
for \gcarb{} and $.162$ for Full-prefix; this is not a no-guard reference.
We report autonomous success as an absolute fraction of test episodes,
without normalizing it by no-guard success.

\begin{table}[t]
\centering
\footnotesize
\setlength{\tabcolsep}{3.5pt}
\begin{tabular}{llrrr}
\toprule
Backbone & $\alpha$ & $\Delta\bar L$ & $\Delta I$ & $\Delta U$ \\
\midrule
Qwen2.5 & .05 & $-.009$ & $-.163$ & $+.102$ \\
 &  & $[-.016,.004]$ & $[-.280,-.102]$ & $[.063,.171]$ \\
Qwen2.5 & .10 & $+.005$ & $-.055$ & $+.029$ \\
 &  & $[.002,.008]$ & $[-.067,-.043]$ & $[.018,.037]$ \\
Qwen3 & .05 & $-.005$ & $+.003$ & $+.000$ \\
 &  & $[-.033,.004]$ & $[-.004,.053]$ & $[-.002,.002]$ \\
Qwen3 & .10 & $-.007$ & $-.289$ & $+.081$ \\
 &  & $[-.024,.016]$ & $[-.427,-.081]$ & $[.030,.118]$ \\
\bottomrule
\end{tabular}
\caption{Paired mean differences and empirical 5th--95th assignment percentiles
over $200$ shared splits per cell. Intervention is lower in three cells,
while the Qwen3 $.05$ cell has essentially no utility difference. These
are assignment-sensitivity summaries, not independent-sample intervals.}
\label{tab:paired-sensitivity}
\end{table}

\subsection{Complete window and random-context controls}

\begin{table}[t]
\centering
\scriptsize
\setlength{\tabcolsep}{3.4pt}
\begin{tabular}{llrrrrrrrr}
\toprule
& & \multicolumn{4}{c}{Fixed-window-4} &
\multicolumn{4}{c}{Equal-size random context} \\
\cmidrule(lr){3-6}\cmidrule(lr){7-10}
Backbone & $\alpha$ & $\bar L$ & $I$ & $U$ & CtxRed &
$\bar L$ & $I$ & $U$ & CtxRed \\
\midrule
\multirow{4}{*}{Qwen2.5}
& .025 & .022 & .869 & .010 & .130 & .022 & .977 & .008 & .494 \\
& .050 & .044 & .669 & .072 & .159 & .035 & .515 & .168 & .565 \\
& .100 & .072 & .395 & .228 & .175 & .077 & .332 & .261 & .596 \\
& .200 & .133 & .000 & .436 & .201 & .133 & .000 & .436 & .580 \\
\midrule
\multirow{4}{*}{Qwen3}
& .025 & .001 & 1.000 & .000 & .003 & .000 & 1.000 & .000 & .089 \\
& .050 & .046 & .989 & .002 & .133 & .041 & .992 & .002 & .511 \\
& .100 & .092 & .818 & .028 & .188 & .085 & .535 & .107 & .590 \\
& .200 & .149 & .369 & .176 & .242 & .159 & .320 & .196 & .646 \\
\bottomrule
\end{tabular}
\caption{Complete context-control grid, averaged over $200$ assignments per cell.
Equal-size random context performs similarly to \gcarb{} on this natural
corpus. Fixed-window-4 uses a different recency-based selection rule.
CtxRed is computed separately on each method's scored proposals using
Equation~\ref{eq:ctxred}.}
\label{tab:context-controls}
\end{table}

Across random-selector seeds $0$--$19$ at $\alpha=.05$, mean random-context
ledger loss is $.035$ for Qwen2.5 and $.041$ for Qwen3, equal to the
corresponding \gcarb{} losses at displayed precision. Random-minus-\gcarb{}
intervention differences range from $-.003$ to $.000$ for Qwen2.5 and
are zero for every Qwen3 seed. These comparisons support context
restriction under the tested scorer without a demonstrated natural-data
advantage for dependency selection. Figure~\ref{fig:frontiers} in the
main text shows the full operating grid.

\section{Constructed suites and held-out diagnostics}
\label{app:suites}

\subsection{Construction}

CES trajectories open with nine neutral filler actions. A harmful
trajectory then reads a private record, summarizes it while carrying its
private handle, and sends the summary and handle to a fixed external
recipient. The clean paired trajectory replaces the private handle with
a public one and performs the same summarize-and-send pattern. Each
harmful trajectory has three near-misses: a public-only send; a private
read and local summary without an external sink; and a private branch
followed by a disconnected public branch whose sink sends only the public
handle. The clean cohort is the harmful/public-send subset of the broader
near-miss cohort: $80$ leaks and $80$ matched non-leaks. The near-miss
cohort contains those same $80$ leaks and all $240$ non-leaks. The cohorts
are therefore not independent replications.

Graph-braid samples $2{,}000$ harm indicators from a seeded Bernoulli
recipe with probability $.4$ and creates a fresh private canary and
private/public object pair for each trajectory. Each trajectory has $20$
actions. Positions 1--3 read the private record, open a benign branch,
and continue private reasoning; position 5 sends an authorized benign
update. The other positions interleave private and benign records with
branch identifiers. A harmful trajectory sends the private canary
externally at position 16; a benign trajectory continues its branch there.
The cohort prevalence is set by construction, not estimated for deployed
agents.

Harm and branch labels come from the generation recipe and are stored
outside every graph and scorer input. Appendix~\ref{app:protocol} states
the fold sizes: $100$ pair-grouped assignments for each CES cohort and $200$ trajectory assignments for
graph-braid. These results describe selectivity of the frozen rule on
these constructed traces.

\subsection{Results and ties}

\begin{table}[t]
\centering
\scriptsize
\setlength{\tabcolsep}{3.1pt}
\begin{tabular}{llrrrrrr}
\toprule
Cohort & Method & $\bar L$ & Harm stop & Benign stop & $I$ & $U$ &
CtxRed \\
\midrule
\multirow{6}{*}{CES clean: 80/80}
& No guard & .500 & .000 & .000 & .000 & 1.000 & -- \\
& Ledger-current & .000 & 1.000 & 1.000 & 1.000 & .000 & .846 \\
& Full-prefix & .000 & 1.000 & .000 & .500 & .500 & .000 \\
& Fixed-window-4 & .000 & 1.000 & .000 & .500 & .500 & .462 \\
& Equal-size random & .000 & 1.000 & 1.000 & 1.000 & .000 & .346 \\
& \gcarb{} & .000 & 1.000 & .000 & .500 & .500 & .346 \\
\midrule
\multirow{6}{*}{CES near-miss: 80/240}
& No guard & .250 & .000 & .000 & .000 & 1.000 & -- \\
& Ledger-current & .000 & 1.000 & .667 & .750 & .250 & .853 \\
& Full-prefix & .000 & 1.000 & .333 & .500 & .500 & .000 \\
& Fixed-window-4 & .000 & 1.000 & .333 & .500 & .500 & .481 \\
& Equal-size random & .000 & 1.000 & .667 & .750 & .250 & .395 \\
& \gcarb{} & .000 & 1.000 & .000 & .250 & .750 & .395 \\
\midrule
\multirow{6}{*}{Graph-braid: 2,000}
& No guard & .390 & .000 & .000 & .000 & 1.000 & -- \\
& Ledger-current & .000 & 1.000 & .000 & .390 & .610 & .898 \\
& Full-prefix & .000 & 1.000 & .095 & 1.000 & .000 & .000 \\
& Fixed-window-4 & .000 & 1.000 & .095 & 1.000 & .000 & .067 \\
& Equal-size random & .000 & 1.000 & .079 & .905 & .095 & .450 \\
& \gcarb{} & .000 & 1.000 & .000 & .390 & .610 & .471 \\
\bottomrule
\end{tabular}
\caption{Constructed diagnostics at $\alpha=.10$: $100$ grouped CES assignments
and $200$ graph-braid assignments. CES labels give harmful/non-leak pool
counts; the graph-braid label gives total pool size. All gated methods
have zero reported mean ledger loss and stop the harmful sinks. Benign
stop divides by non-leak trajectories in CES, but by all trajectories
in graph-braid (Appendix~\ref{app:protocol}). These are descriptive
held-out results, not a CES certificate under the episode-permutation
corollary.}
\label{tab:constructed}
\end{table}

On CES clean, Full-prefix and Fixed-window-4 match \gcarb{} on loss,
intervention, and utility. On near-misses, \gcarb{} avoids benign stops
where Full-prefix and Fixed-window-4 stop one third and Ledger-current
and equal-size random stop two thirds. The equal-size comparison supports
a role for which records are retained beyond their count in these
constructed cohorts under this scorer.

On graph-braid, Ledger-current matches \gcarb{} on loss, intervention,
and utility. Equal-size random has benign-branch first stops on $.079$
of all trajectories and intervention $.905$, compared with $.000$ and
$.390$ for \gcarb{}. Random and \gcarb{} are size-matched at a given
decision but have CtxRed $.450$ and $.471$: different stopping points
change the scored-proposal sets in Equation~\ref{eq:ctxred}. These ties
and differences limit the claim to the comparisons actually distinguished
by each construction.

\section{Observable-evidence and perturbation diagnostics}
\label{app:observability}

\subsection{Fixed-threshold evidence removal}

The $2\times2$ audit reuses $80$ matched CES pairs and $100$ grouped
assignments, crossing structured provenance present/absent with exact
sink text present/laundered. In each assignment, the threshold is
calibrated on the provenance-present, text-laundered cell and then held
fixed across all four cells. Evidence variants alter the graph/scorer
view while the harm labels and declared evaluation target stay fixed.

\begin{table}[t]
\centering
\footnotesize
\setlength{\tabcolsep}{4pt}
\begin{tabular}{lrrrr}
\toprule
Evidence & $\bar L$ & Path recall & False path & CtxRed \\
\midrule
Provenance $+$, text $+$ & .000 & 1.000 & .000 & .346 \\
Provenance $+$, text $-$ & .000 & 1.000 & .000 & .346 \\
Provenance $-$, text $+$ & .000 & 1.000 & .000 & .833 \\
Provenance $-$, text $-$ & .500 & .000 & .000 & .846 \\
\bottomrule
\end{tabular}
\caption{Fixed-threshold CES evidence audit. ``Text $-$'' denotes laundered exact
sink text and ``provenance $-$'' removal of structured provenance from
the policy evidence. The threshold is calibrated in the provenance $+$,
text $-$ cell and reused. Removing both channels gives zero harmful-path
recall and mean ledger loss $.500$, equal to the no-guard loss of the
balanced cohort. This is an observability/shift diagnostic, not a
four-cell conformal certificate.}
\label{tab:evidence}
\end{table}

Path recall is measured on harmful source-to-sink examples, and false
path on the matched non-leak examples. Either remaining observable
channel retains all harmful paths in this construction; absence of both
removes them. Keeping the ledger fixed preserves the evaluation target,
while changed evidence can still increase the loss of a frozen policy.
Figure~\ref{fig:appendix-observability}A displays the path-recall results.

\subsection{Graph-braid perturbations with recalibration}
\label{app:perturbations}

The graph-braid sweep varies harm prevalence, benign-branch density,
additive score noise, and observable-edge dropout at $\alpha=.10$.
Each condition receives its own calibrated threshold against the same
global ledger. This measures the risk--intervention trade-off after
recalibration. The CES audit above instead reuses a fixed threshold
across evidence conditions.

The sweep contains $16$ conditions with $20$ trajectory assignments per
condition. For condition $c\in\{0,\ldots,15\}$, the split seeds are
$20260609+100c$ through $20260609+100c+19$. The listed prevalence levels
are $.1,.2,.6,.8$; benign-branch densities are $.25,.75,.90$; score-noise
levels are $.05,.10,.20$; and edge-dropout levels are $.05,.10,.20$.
Two further conditions combine moderate and severe perturbations.
We report the sweep as a descriptive sensitivity analysis of the
constructed traces, with separate calibration in each condition.

\begin{table}[t]
\centering
\scriptsize
\setlength{\tabcolsep}{2pt}
\begin{tabular}{lrrrrr}
\toprule
Sweep & Conditions & Max $\bar L$ & Min H-stop & Max B-stop &
CtxRed range \\
\midrule
Baseline & 1 & .000 & 1.000 & .000 & .471--.471 \\
Harm prevalence & 4 & .054 & .400 & .000 & .471--.472 \\
Benign density & 3 & .000 & 1.000 & .000 & .471--.471 \\
Score noise & 3 & .097 & .765 & .030 & .469--.472 \\
Edge dropout & 3 & .000 & 1.000 & .094 & .400--.471 \\
Combined stress & 2 & .080 & .869 & .063 & .459--.467 \\
\bottomrule
\end{tabular}
\caption{Graph-braid sensitivity at $\alpha=.10$. Each condition's
metrics are averaged over $20$ assignments; each row gives extrema over
the listed conditions. H-stop is the fraction of harmful sinks blocked.
B-stop counts first stops on benign branches, divided by all trajectories.
The reported mean losses stay below the budget; these assignment
averages do not guarantee that every held-out batch meets it.}
\label{tab:robustness}
\end{table}

Figure~\ref{fig:appendix-observability}B shows the loss and intervention
means. At $20\%$ edge dropout, intervention reaches $1.0$: every episode
is stopped. This illustrates the possible utility cost of recalibration
when dependency evidence is degraded. The reported standard errors
summarize variation over repeated assignments of constructed data;
they provide no confidence guarantee for a new deployment population.
The sweep leaves fixed-threshold transfer under distribution shift
untested.

\begin{figure}[t]
\centering
\safeincludegraphics[width=\linewidth]{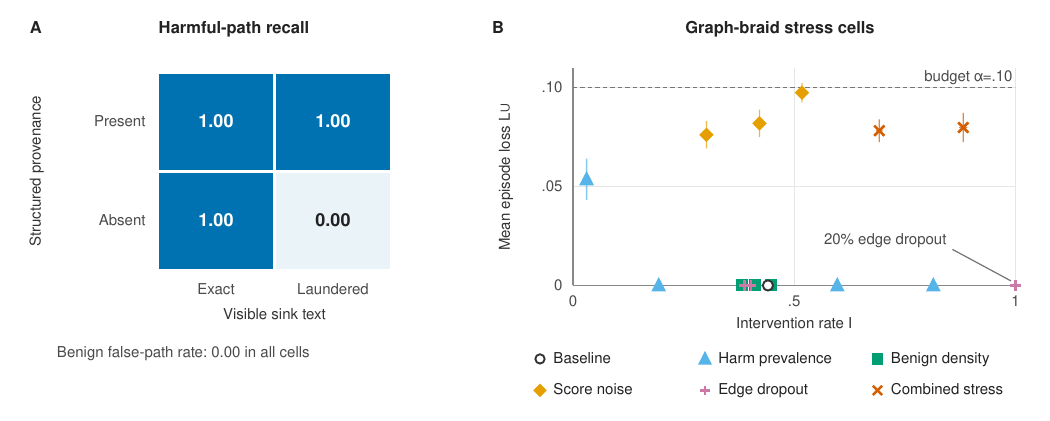}
\caption{\textbf{Evidence removal and recalibrated perturbation results.}
\textbf{A:} Harmful-path recall over $100$ grouped CES assignments of
$80$ matched pairs. The threshold is calibrated with structured
provenance present and exact sink text laundered, then reused in all
four cells. Removing both channels reduces recall to zero; the benign
false-path rate is zero throughout. Table~\ref{tab:evidence} gives ledger
losses. \textbf{B:} Mean episode loss and intervention for $16$ graph-braid
conditions at budget $\alpha=.10$, with separate calibration in each
condition. Whiskers show standard errors over $20$ assignments per
condition. The $20\%$ edge-dropout condition stops every episode.
These recalibrated comparisons do not test transfer of a fixed threshold
across conditions.}
\label{fig:appendix-observability}
\end{figure}

\section{Online interface check}
\label{app:integration}

Our Qwen2.5 integration check applies the split-$0$ \gcarb{}
threshold $.58$ at $\alpha=.05$ across four AgentDojo suites. Its primary
scope contains $48$ attacked and $16$ no-attack episodes. Online scores
and decisions agree with deterministic same-code recomputation on all
$135$ primary proposals, or $152$ proposals after $12$ auxiliary controls
are included. Every current node has an empty result at scoring time.
The wrapper stops $22/48$ attacked and $7/16$ no-attack episodes;
autonomous success is $13/48$ and $4/16$, respectively, and one attacked
episode reaches a registered ledger event.

These counts check consistency between the online interface and same-code
recomputation. The selected episodes do not establish exchangeability
with a deployment population, so this is an integration diagnostic.
Independent ledger validation and population risk evaluation require
separate evidence.

\section{Ethics and broader impacts}
\label{app:ethics}

The intended benefit is to make an agent gate's risk target explicit and
reduce avoidable interventions while controlling declared harm in
expectation. Misinterpreting this marginal guarantee as a promise for
every episode, or treating an incomplete ledger as comprehensive safety
coverage, could encourage unsafe deployment. Unnecessary stops can also
disrupt useful work and transfer burdens to human operators; our
experiments do not evaluate outcomes after handoff. All experiments use
a sandboxed benchmark with synthetic records and canary strings, without
real private data. Deployment requires validating the ledger and evidence
channels for the intended setting and evaluating the handoff process.

\end{document}